\documentclass[11pt]{article}
\usepackage{url}
\usepackage{bbm}
\usepackage{xfrac}
\usepackage{mathtools}
\usepackage{color}
\usepackage{amsmath,amsthm,amssymb,amsfonts}
\usepackage{bm,epsfig,epsf,url,dsfont}
\usepackage{fullpage}
\usepackage[top=1in,bottom=1in,left=1in,right=1in]{geometry}
\usepackage{graphicx}
\usepackage{subfigure}

\usepackage{hyperref} 
\hypersetup{colorlinks=true, 
linkcolor=blue,
citecolor=blue,
filecolor=black,
urlcolor=black}

\usepackage{siunitx}

\usepackage{comment}

\usepackage[mathcal]{eucal}

\usepackage[round]{natbib}

\newcommand{\eff}{{\mathsf{eff}}}
\newcommand{\gps}{{\mathsf{GPS}}}

\newcommand{\bv}{{\big\Vert}}

\usepackage{bm}
\usepackage{dirtytalk}

\newtheorem{definition}{Definition}

\newtheorem{proposition}{Proposition}

\newtheorem{corollary}{Corollary}

\newcommand{\upstairs}[1]{\textsuperscript{#1}}
\newcommand{\affilone}{1}
\newcommand{\affiltwo}{2}
\newcommand{\affilthree}{3}
\newcommand{\affilfour}{4}

\usepackage{parskip}

\newcommand{\y}{\mathbf{y}}
\newcommand{\Y}{\mathbf{Y}}
\newcommand{\g}{\mathbf{g}}
\newcommand{\F}{\mathbf{F}}

\def\<{\begin{equation}}
\def\>{\end{equation}}

\usepackage{footnote}
\makesavenoteenv{tabular}

\begin{document}


\begin{center}

{\Large \bf Universal Smoothed Score Functions for Generative Modeling}

  \vspace*{.2in}
  
  \begin{tabular}{ccc}
  
    Saeed Saremi\upstairs{\affilone, \affiltwo}\quad 
    Rupesh Kumar Srivastava\upstairs{\affilthree}\quad    
    Francis Bach\upstairs{\affilfour}\quad
   
   \vspace*{.1in} \\
   
    \upstairs{\affilone}UC Berkeley
    \upstairs{\affiltwo}Prescient Design, Genentech, Roche 
    \upstairs{\affilthree}NNAISENSE \\
    \upstairs{\affilfour}Inria, Ecole Normale Sup\'erieure, PSL Research University
  \end{tabular}
  
  \vspace*{.2in}

\begin{abstract}
   We consider the problem of generative modeling based on smoothing an unknown density of interest in $\mathbb{R}^d$ using factorial kernels with $M$ independent Gaussian channels with equal noise levels introduced by~\citet{saremi2022multimeasurement}. First, we fully characterize the time complexity of learning the resulting smoothed density in $\mathbb{R}^{Md}$, called M-density, by deriving a universal form for its parametrization in which the score function is by construction permutation equivariant. Next, we study the time complexity of sampling an M-density by analyzing its condition number for Gaussian distributions. This spectral analysis gives a geometric insight on the ``shape'' of M-densities as one increases $M$. Finally, we present results on the sample quality in this class of generative models on the CIFAR-10 dataset where we report Fr\'echet inception distances (14.15), notably obtained with a single noise level on long-run fast-mixing MCMC chains.
\end{abstract}
  
\end{center}

\section{Introduction}

Smoothing a density with a kernel is a technique in nonparametric density estimation that goes back to~\citet{parzen1962estimation} at the birth of modern statistics. There has been a recent interest in estimating smoothed densities that are obtained by Gaussian convolution~\citep{saremi2019neural, goldfeld2020convergence}, where instead of the random variable $X$ the problem is to model the random variable $Y=X+\mathcal{N}(0,\sigma^2 I_d)$ based on a finite number of independent samples $\{x_i\}_{i=1}^n$ drawn from $p_X$. There is a subtle difference between this problem and the problem addressed by~\citet{parzen1962estimation}. In the original problem (estimating $p_X$),  the bandwidth of the Gaussian kernel $\sigma$ is adjusted depending on the number of samples (typically tending to zero when the sample size goes to infinity). Here the kernel bandwidth $\sigma$ is \emph{fixed}. As one might expect, learning $p_Y$ is simpler than learning $p_X$, but the problem is quite rich with deep connections to empirical Bayes and score matching as we highlight next.

Empirical Bayes formulated by~\citet{robbins1956empirical} is concerned with the problem of estimating the random variable~$X$ given a single noisy observation $Y=y$ assuming the noise model $p_{Y|X}$ is known. A classical result states that the least-squares estimator of $X$ is the Bayes estimator:
$$ \widehat{x}(y) = \frac{\int x p(y|x) p(x) dx }{\int p(y|x) p(x) dx}.$$ But the remarkable result obtained in Robbin's seminal paper is that the Bayes estimator can be written in closed form purely in terms of $p_Y$ for a variety of noise models: the explicit knowledge of $p_X$ is not required in this estimation problem. In addition, this dependency is only in terms of the \emph{unnormalized} $p_Y$ for all known empirical Bayes estimators, although this fact was not highlighted by~\citet{robbins1956empirical}.  For isotropic Gaussian the estimator takes the form\footnote{This 
result for Gaussian noise models was first derived by~\citet{miyasawa1961empirical} but has a rich history of its own; see~\citet{raphan2011least} for a survey.} $$\widehat{x}(y) = y + \sigma^2 g(y),$$
where $g(y)= \nabla \log p(y)$ is known in the literature as the score function~\citep{hyvarinen2005estimation}.

This is the starting point in neural empirical Bayes (NEB)~\citep{saremi2019neural}, where the score function is parametrized using a neural network, arriving at the following learning objective
$$ \mathcal{L}(\theta) = \mathbb{E}_{(x,y)\sim p(y|x) p(x)} \Vert x-\widehat{x}_\theta(y)\Vert^2.$$
 Algorithmically, NEB is attractive on two fronts: (i) the learning/estimation problem is reduced to the optimization of the least-squares denoising objective, where MCMC sampling is not required during learning, (ii) generative modeling is reduced to sampling $p_Y$ which is better conditioned than sampling $p_X$ (see~\autoref{sec:mixing}) combined with the estimation of $X$ which is a deterministic computation itself, referred to as walk-jump sampling (WJS).
 
 The main problem with NEB, from the perspective of generative modeling (sampling $p_X$), is that we cannot sample from $p(x|y)$, and we do not have a control over how concentrated $p(x|y)$ is around its mean $\widehat{x}(y)=\mathbb{E}[X|Y=y].$ A solution to this problem was formulated by~\citet{saremi2022multimeasurement}, where the noise model  in NEB was replaced with a multimeasurement noise model (MNM):
 $$ p(\y|x) = \prod_{m=1}^M p(y_m|x),$$
 where the bold-faced $\y$ denotes the multimeasurement random variable $\y=(y_1,\dots,y_M)$. As we review in \autoref{sec:background}, the algorithmic attractions of NEB carry over to Gaussian MNMs with the added benefits that by simply increasing the number of measurements $M$, the posterior $p(x|\y)$ automatically concentrates around its mean. Of particular interest is the case where the $M$ noise levels are identical, therefore the M-density $p(\y)$ is permutation invariant\textemdash this class of models is denoted by $(\sigma, M)$ which is our focus in this paper. 
 \subsection{Contributions}
 Our theoretical contributions are concerned with answering the following two questions:
 \begin{itemize}
 	\item \emph{What is the time complexity of learning M-densities?}  We show that the M-densities associated with $(\sigma, M)$ and $(\sigma', M')$ can be mapped to each other if $\sigma/\sqrt{M} = \sigma'/\sqrt{M'}$. The permutation-invariant Gaussian M-densities are therefore grouped into universality classes $[\sigma_\eff]$ where $\sigma_\eff \coloneqq \sigma/\sqrt{M}$.
We arrive at a parametrization scheme for the score function associated with M-densities, called $\gps$, that is by construction permutation equivariant with respect to the permutation of the measurement indices. As a side effect of the $\gps$ parametrization, we derive a single estimator of $X$ given $\Y=\y$ instead of $M$ (approximately equal) estimators by~\citet{saremi2022multimeasurement}.
 	\item \emph{What is the time complexity for sampling M-densities?} Knowing that M-densities are grouped into universality classes, the more subtle question is: which member has better mixing time properties? This is an important question to answer in understanding the generative modeling properties of M-densities. This question is a difficult one in its full generality, but to shed light on it we assume the original density $p_X$ is a non-isotropic Gaussian, and we study the full spectrum of the corresponding M-density. The calculation gives insight on the ``geometry'' of M-densities as one increases $M$. See  \autoref{fig:schematic} for a schematic.
 \end{itemize}
 \clearpage
 Experiments are focused on the generative modeling problem on the CIFAR-10 dataset~\citep{krizhevsky2009learning}. This dataset has proved to be challenging for generative models due to the diversity of image classes present. The performance of the generative models are measured in FID score~\citep{heusel2017gans} (the lower score is better). As an example, a sophisticated model like BigGAN~\citep{brock2018large}, in the generative adversarial networks \citep{goodfellow2014generative} family, achieves the FID score of 14.73. Our framework is based on a simple denoising objective with a single noise level, yet despite its simple structure we can achieve  the FID score of {\bf 14.15}, which is remarkable in this class of models. Our experimental results question the current perception in the field that denoising models with a single noise level cannot be good generative models. 
     
 \begin{figure}[t!] 
\begin{center}
\hspace{1cm} \begin{subfigure}[$p(x)$]
 {\includegraphics[width=0.32\textwidth]{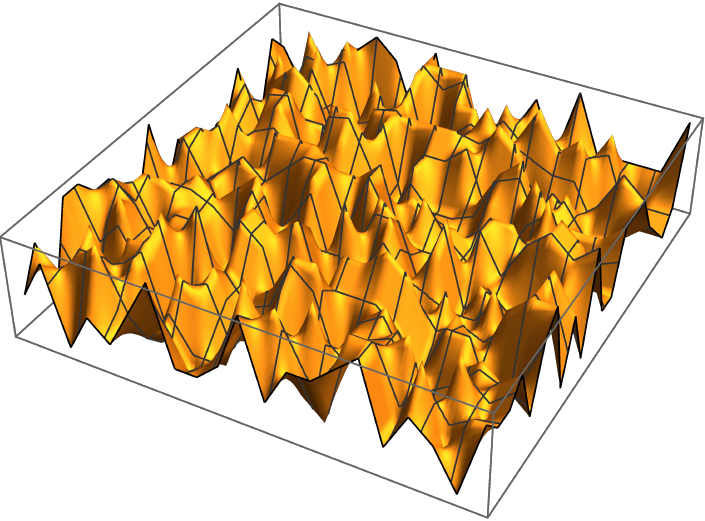}}
\end{subfigure}
\begin{subfigure}[$p(\y)$]
 {\includegraphics[width=0.24\textwidth]{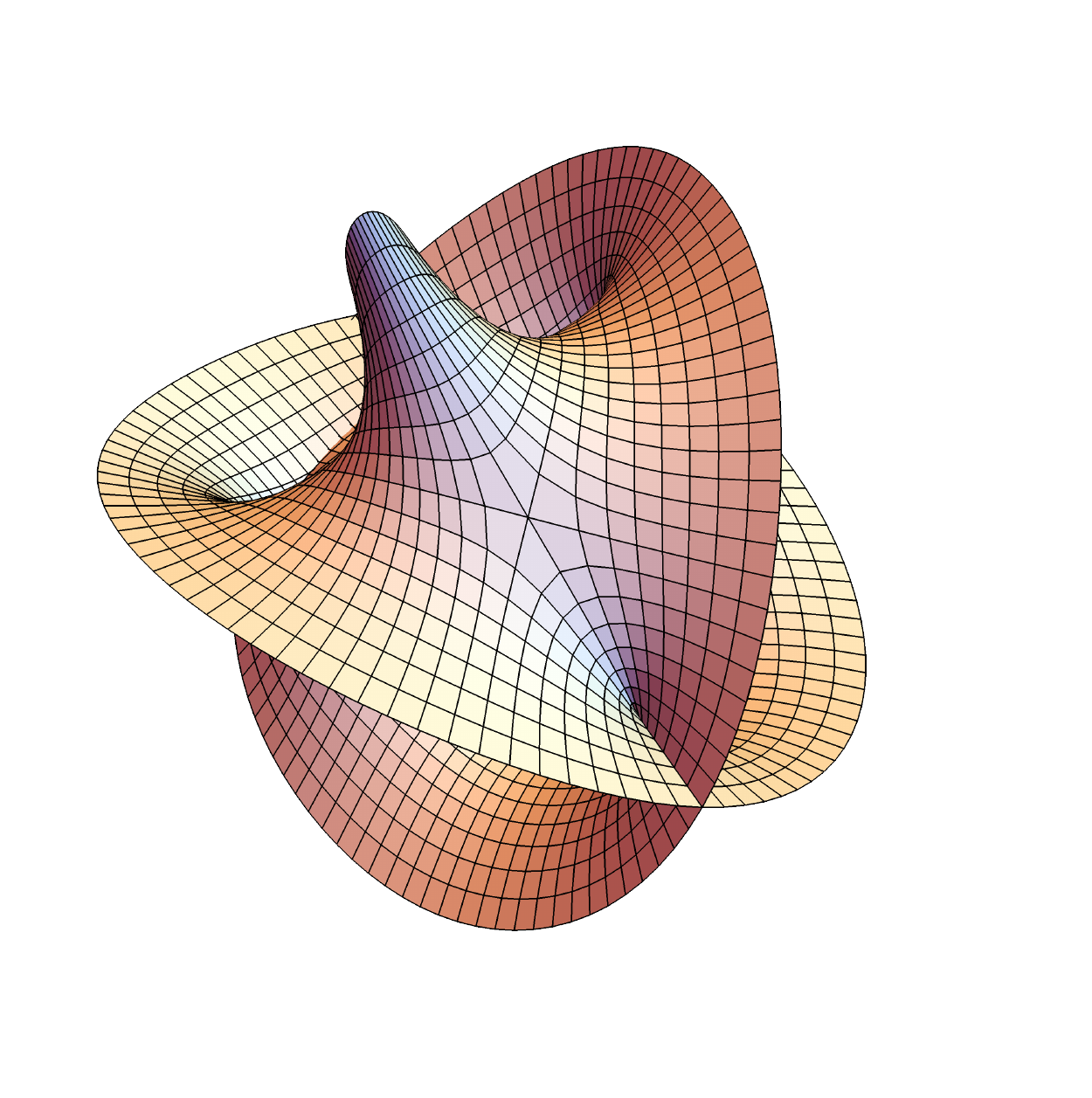}}
\end{subfigure}
\end{center}
\caption{({\it The geometry of M-densities}) (a) Schematic of a complex density in $\mathbb{R}^d$ ($d=2$). (b) The plot represents the manifold associated with the corresponding permutation-invariant M-density in  $\mathbb{R}^{Md}$. The schematic is meant to capture the fact that the M-density is symmetric and it is smoother than the original density. } 
\label{fig:schematic}
 \end{figure}
      
\subsection{Related Work}
\label{sec:related}

The work directly related to this paper is by \citet{saremi2022multimeasurement} who introduced a generative modeling framework based on smoothing an unknown density of interest with factorial kernels with $M$ channels. Our focus here is on Gaussian kernels and our main contribution is to show any single measurement model $(\sigma, 1)$ can be mapped to a multimeasurement one $(\sigma \sqrt{M}, M)$; no additional learning is required. In particular, in our work the neural network inputs are in $\mathbb{R}^d$ as opposed to $\mathbb{R}^{Md}$ in the earlier work. In addition, there were open questions regarding the role $M$ in the sampling complexity in the earlier work that this work aims to address.

At a broader level, this work is related to the research on denoising density models, a class of probabilistic models that grew from the literature on score matching and denoising autoencoders~\citep{hyvarinen2005estimation, vincent2011connection, alain2014regularized, saremi2018deep}. These models were not successful in the past on challenging generative modeling tasks, e.g. on the CIFAR-10 dataset, which in turn led to research on denoising objectives with multiple noise levels~\citep{song2019generative}. In our experiments, we revisit denoising density models with a single noise level. In particular, our experimental results question the current perception in the field around the topic of single vs. multiple noise scales. As an example, we point out that the FID score in this paper (\num{14.15}) is significantly lower than the one reported by~\citet{song2019generative} (\num{25.32}) which was obtained using annealed Langevin MCMC with multiple noise scales. Also see \citet{jain2022journey} for a recent work on score-based generative modeling with a single noise level. 
 
\section{Background} \label{sec:background}
In this section, we review smoothing with factorial kernels and their use for generative modeling. We refer to~\citet{saremi2022multimeasurement} for more details and references.
\paragraph{Factorial Kernels.} Smoothing a density  with a (Gaussian) kernel is a well-known technique in nonparametric density estimation that goes back to~\citet{parzen1962estimation}; see~\citet{hastie2009elements} for a general introduction to kernels. Given a density $p_X$ in $\mathbb{R}^d$ one can construct a smoother density $p_Y$ in $\mathbb{R}^d$ by convolving it with a positive-definite kernel $k: \mathbb{R}^d \times \mathbb{R}^d \rightarrow \mathbb{R} $:
$$ p(y) = \int k(x,y) p(x) dx.$$
In this paper we only consider (translation-invariant) isotropic Gaussian kernels, where one can also have a dual perspective on the kernel as the conditional density $k(x,y)=p(y|x)$, where
$$ p(y|x) = \frac{1}{Z(\sigma)}\exp\left( - \frac{\Vert y-x\Vert^2}{2\sigma^2}\right) \eqqcolon \mathcal{N}(y;x,\sigma^2 I_d). $$
Here $Z(\sigma)$ is the partition function associated with the isotropic Gaussian. From this angle, smoothing a density with an isotropic Gaussian kernel can also be expressed in terms of random variables as follows:
$$ Y = X + \mathcal{N}(0,\sigma^2 I_d).$$
A factorial kernel is a generalization of the above where the kernel $k(x,y)$ takes the following factorial form with $M$ kernel components:
$$ k(x,\y) = \prod_{m=1}^M k(x,y_m)\text{, where } \y=(y_1,\dots,y_M).$$
For isotropic Gaussian kernels this is equivalent to $Y_m = X + \mathcal{N}(0,\sigma^2 I_d)$ for $m \in [M]$ (with independent Gaussians). The kernel $k(x,\y)$ is referred to as multimeasurement noise model (MNM).\footnote{One can view smoothing with a factorial kernel in the context  of a communication system~\citep{shannon1948mathematical} with $M$ independent noise channels, i.e., given $X=x$ samples $(y_1,\dots,y_m)$ are obtained by adding $M$ independent  isotropic Gaussian noise to $x$.} The result of the convolution with a Gaussian MNM is referred to as Gaussian M-density associated with the random variable $\Y=(Y_1,\dots,Y_M)$ which takes values in $\mathbb{R}^{Md}$. Given $\{x_i\}_{i=1}^n$, independent draws from the density $p_X$, we are interested in estimating the associated M-density for a given fixed noise level $\sigma$.
\paragraph{Multimeasurement Bayes Estimators.} How can we go about estimating the M-density? One approach is to formulate it as a learning problem~\citep{vapnik1999nature} by parametrizing the M-density (say with a neural network) and devising an appropriate learning objective. For $M=1$, the learning objective is given by:
\< \label{eq:objective} \mathcal{L}(\theta) = \mathbb{E}_{(x,y)\sim p(x) p(y|x)} \Vert x - \widehat{x}_\theta(y)\Vert^2,\>
where $\widehat{x}_\theta(y)$ is a parametrization of the Bayes estimator of $X$ given $Y=y$ in terms of $p_\theta(y)$. This approach heavily relies on the fact that the Bayes estimator $\widehat{x}(y)$ can indeed be expressed in closed form in terms of $p(y)$,  which is a key result at the heart of empirical Bayes~\citep{robbins1956empirical}. In addition, for Gaussian kernels, the Bayes estimator can be expressed in terms of the score function $\nabla \log p(y)$, a result that goes back to~\citet{miyasawa1961empirical}, therefore for learning $p(y)$ one can ignore its partition function~\citep{saremi2019neural}. This key result in the empirical Bayes literature was extended by~\citet{saremi2022multimeasurement} to Poisson and Gaussian MNMs, where for Gaussian MNMs, $\widehat{x}(\y)=\mathbb{E}[X|\Y=\y]$ takes the following form:
\< \label{eq:xhat_m} \widehat{x}(\y) = y_m + \sigma^2 \nabla_m \log p(\y),\>
where $m \in [M]$ is an arbitrary measurement index (the result is invariant to this choice). Note that the Bayes estimator $\widehat{x}(\y)$ only depends on the score function associated with the M-density, therefore one can use~\autoref{eq:objective} as the objective for learning the energy/score function associated with the M-density by simply replacing $y$ with $\y=(y_1,\dots,y_M)$.

\paragraph{Walk-Jump Sampling.} Following learning the score function $\nabla \log p(\y)$, one can use Langevin MCMC to draw exact samples from $p(\y)$. For $M=1$ the density $p(y)$ is smoother than $p(x)$ and MCMC is assured to mix faster. What about drawing samples from $p(x)$? The idea behind walk-jump sampling (WJS) is that one can indeed use the score function $\nabla \log p(y)$ to estimate $X$ thus arriving at approximate samples from $p(x)$~\citep{saremi2019neural}. There is clearly a trade-off here: by decreasing $\sigma$ the estimate of $X$ becomes more and more accurate (WJS becomes more and more exact) but this comes at the cost of sampling a less smooth $p(y)$ where MCMC has a harder time. On the surface, what is intriguing about M-density is that one can keep $\sigma$ to be ``large'' and still have a control on generating exact samples from $p(x)$ by simply increasing $M$. The full picture on the effects of increasing $M$ is more complex which we discuss after our analysis in \autoref{sec:mixing}.

\section{Universal M-densities}\label{sec:universal}
In this section we derive a general  expression for the M-density $p(\y)$ and the  score function $\nabla \log p(\y)$ for Gaussian MNMs with equal noise levels $\sigma$. For equal noise levels, the M-density (resp. score function) is permutation invariant (resp. equivariant) under the permutation of measurement indices~\citep{saremi2022multimeasurement}. However, it is not clear a priori how this invariance/equivariance should be reflected in the parametrization. The calculation below clarifies this issue, where we arrive at a general permutation invariant (resp. equivariant) parametrization for the M-density (resp. score function) in which the empirical mean of the $M$ measurements
$$ \overline{y} =  M^{-1} \sum_{m=1}^M y_m $$
plays a central role.  We start with a rewriting of the log p.d.f. of the factorial kernel:
 \<
 \begin{split}
 	&- 2 \sigma^2 \log p(\y|x) 
 	= \sum_{m=1}^M \bv y_m - x \Vert ^2 + C \\	
 	&= M \Vert x \Vert^2 -2 \langle \sum_{m=1}^M y_m, x\rangle + \sum_{m=1}^M \Vert y_m \Vert^2 + C \\
 	&= M (\Vert x \Vert^2 - 2 \langle \overline{y}, x \rangle + \overline{\Vert y\Vert^2})+ C \\
 	&= M(\Vert x - \overline{y}\Vert^2 + \overline{\Vert y\Vert^2} -\Vert \overline{y} \Vert^2 )+ C,
 \end{split}
 \>
 where $C=2 \sigma^2 M \log Z(\sigma)$ and $\overline{\Vert y\Vert^2} $ is short for
$$
 \overline{\Vert y\Vert^2}  = M^{-1} \sum_{m=1}^M \bv y_m \bv^2.
$$
Now, we view the smoothing kernel $p(\y|x)$ as a Gaussian distribution over $X$ centered at $\overline{y}$:
 \< \begin{split}
 	& \log p(\y) = \log \int p(\y|x) p(x) dx \\
 	&= \log \int \mathcal{N}(x; \overline{y}, \sigma_\eff^2 I_d) p(x) dx + \frac{\Vert \overline{y} \Vert^2- \overline{\Vert y\Vert^2}}{2\sigma_\eff^2} + C' \\
 	&= \log \mathbb{E}_{X\sim \mathcal{N}(\overline{y},\sigma_\eff^2 I_d)} [p(X)]+\frac{\Vert \overline{y} \Vert^2- \overline{\Vert y\Vert^2}}{2\sigma_\eff^2}+C',
 \end{split}
   \>
where $\sigma_{\eff} \coloneqq \sigma/\sqrt{M} $, and $C'=\log Z(\sigma_\eff) - M \log Z(\sigma)$. Note that the first term above is a function of $\overline{y}$ for any distribution~$p_X$, therefore the energy function $f(\y)$ takes the following form:
\< \label{eq:f} f(\y) =  \frac{1}{2\sigma_\eff^{2}}  \left(  \overline{\Vert y\Vert^2}- \Vert \overline{y}\Vert^2 \right) + \varphi(\overline{y}). \>
Next we consider the functional form for the score function  $\g(\y)=-\nabla f(\y)$. Taking gradients leads to:
$$ \label{eq:g} g_m(\y) =  \frac{1}{2\sigma_\eff^{2}}  \left(  2\overline{y}/M  - 2 y_m/M \right) - \nu(\overline{y}) /M ,$$
where $\nu=\nabla \varphi.$ The expression above is written more compactly (to be used below) as
\<\label{eq:g} \sigma^2 g_m(\y)=(\overline{y}-y_m)- \sigma_{\eff}^2 \cdot \nu(\overline{y}).  \>
Finally, the expression for the Bayes estimator $\widehat{x}(\y)$ is derived (combine \autoref{eq:xhat_m} and \autoref{eq:g}):
\< \label{eq:xhat} \widehat{x}(\y) = \overline{y} - \sigma_\eff^2 \cdot \nu(\overline{y}).\>

\subsection{$\gps$ Parametrization}
The results above leads to the following parametrization for the score function:
\newline
\begin{definition}[$\gps$] \label{def:gps} The $\gps$ parametrization of the score function associated with the M-density for $(\sigma, M)$ models is given by (replace $\nu$ with $\nu_\theta$ in \autoref{eq:g})
\< \label{eq:gps} \sigma^2 g_m(\y; \theta)=(\overline{y}-y_m)- \sigma_{\eff}^2 \cdot \nu_\theta(\overline{y}), \>
where $\nu_\theta: \mathbb{R}^d \rightarrow \mathbb{R}^d$  is either parametrized directly/explicitly or  indirectly/implicitly by parametrizing the function $\varphi_\theta: \mathbb{R}^d \rightarrow \mathbb{R}$. In the later case, $\nu_\theta$ is derived as follows:
$$ \nu_\theta = \nabla \varphi_\theta.$$	
\end{definition}
The $\gps$ parametrization has two important properties captured by the following propositions:
\newline
\begin{proposition}[Permutation equivariance property of $\gps$] \label{prop:equivariance} The score function parametrized in $\gps$ is permutation equivariant:
	\< 	 \g_\theta(\pi(\y)) = \pi(\g_\theta(\y)), \>
where $\pi:[M] \rightarrow [M]$ is a permutation of the noise/measurement channels whose action on $\y=(y_1,\dots,y_M)$ and $\g = (g_1,\dots, g_M)$ is to permute the measurement channels:
\begin{align*} \pi((y_1,\dots,y_M)) &= (y_{\pi(1)},\dots,y_{\pi(M)}).	\\
\pi((g_1,\dots,g_M)) &= (g_{\pi(1)},\dots,g_{\pi(M)}).	
\end{align*}

\end{proposition}
\begin{proof}
	The proof is straightforward since $\overline{y}$ in $\gps$ is permutation invariant.
\end{proof} 
The naming $\gps$ has been derived from the statement of \autoref{prop:equivariance}: $\g_\theta$  is a permutation-equivariant score function. Informally (alluding to GPS as the ``global positioning system''), in the $\gps$ parametrization, the \say{coordinates} of the M-density manifold in high dimensions is validly encoded in the sense of respecting its permutation invariance. This important symmetry is broken in the MDAE parametrization studied in~\citep{saremi2022multimeasurement}. 

In addition, in $\gps$, the Bayes estimator (\autoref{eq:xhat}) takes the following parametric form:
\< \label{eq:xhat_theta} \widehat{x}_\theta(\y) = \overline{y}-\sigma_\eff^2 \cdot \nu_\theta(\overline{y}),\>
where the measurement index $m$ does not appear in the final expression\textemdash this is in contrast to the parametrization studied by~\citet{saremi2022multimeasurement}. 

\subsection{$\gps$ is Universal}

Before formalizing the universality of the $\gps$ parametrization in \autoref{prop:universal} below, we define the notion of universality classes associated with M-densities: \newline
\begin{definition}[M-density Universality Classes] \label{def:universality} We define the universality class $[\sigma_\eff]$ as the set of all $(\sigma, M)$ models 
$$ [\sigma_\eff] := \{(\sigma_1,M_1), (\sigma_2,M_2), \dots\},$$
such that for all $(\sigma_i,M_i) \in [\sigma_\eff]$: $$\frac{\sigma_i}{\sqrt{M_i}} =\sigma_\eff.$$ 
In particular, the models $\{(\sigma_\eff \sqrt{M}, M): M \in \mathbb{N}\}$ belong to the universality class $[\sigma_\eff]$. 	
\end{definition}

The universality property of $\gps$ is captured by the following proposition:
\newline
\begin{proposition}[The universal property of $\gps$] \label{prop:universal} For any parameter $\theta$, all $(\sigma, M)$ models that belong to the same universality class and parametrized by $\gps$ are identical in the sense that they incur the same loss  
\< \label{eq:gps-unversality} \mathcal{L}_{\sigma,M}(\theta) = \mathcal{L}_{\sigma',M'}(\theta) \text{ if } \frac{\sigma}{\sqrt{M}}= \frac{\sigma'}{\sqrt{M'}}, \>
where 
$$ \mathcal{L}_{\sigma,M}(\theta)= \mathbb{E}_{(x,\y)\sim p(x) p(\y|x)} \bv x - \widehat{x}_\theta(\y) \bv^2.$$ 
\end{proposition}
\begin{proof}
	Using \autoref{eq:xhat_theta}, we have
	\< \label{eq:loss} \mathcal{L}_{\sigma, M}(\theta)= \mathbb{E}_{(x,\y)\sim p(x) p(\y|x)} \bv x - \overline{y}+ \sigma_\eff^2 \cdot \nu_\theta(\overline{y})\bv^2.\>
Note that $x - \overline{y}$ has the same law as $\mathcal{N}(0,\sigma_\eff^2 I_d)$, where $\sigma_\eff = \sigma/\sqrt{M}$. Therefore, the learning objective has the interpretation that $\nu_\theta$ makes predictions on the residual noise left in the empirical mean of the noisy measurements since $\overline{y}=x+\bar{\gamma}$, where $\bar{\gamma}=M^{-1}\sum_{m=1}^M \gamma_m$ and $\gamma_m$ are independent samples from $\mathcal{N}(0,\sigma^2 I_d)$. This observation can be made explicit by rewriting $\mathcal{L}_{\sigma, M}(\theta)$ as:
$$ \mathbb{E}_{x\sim p(x), \{\gamma_m \sim \mathcal{N}(0,\sigma^2 I_d)\}_{m=1}^M} \bv \overline{\gamma} - \sigma_\eff^2 \cdot \nu_\theta(x + \overline{\gamma}) \bv^2.$$
The statement of the proposition follows since $\bar{\gamma}$ has the same law as $\mathcal{N}(0,\sigma_{\eff}^2I_d)$ for $(\sigma,M)$ and $(\sigma',M')$ models since they are both in the universality class $[\sigma_\eff]$.\end{proof}

\begin{corollary} \label{corollary:xhatlaw} The laws of $\widehat{x}_\theta(\Y)$ are identical for all $(\sigma, M)$ models in the same universality class and for all $\theta$ in the $\gps$ parametrization.
\end{corollary}
\begin{proof}
In the $\gps$ parametrization $\widehat{x}_\theta(\Y) = \overline{Y}-\sigma_\eff^2 \cdot \nu_\theta(\overline{Y})$ (\autoref{eq:xhat_theta}). The proof then follows from the proof in the above proposition since $\overline{Y}$ has the same law as $X+\mathcal{N}(0,\sigma_\eff^2 I_d) $ for all models in $[\sigma_\eff]$.
	\end{proof}
\section{On the shape of M-densities} \label{sec:mixing}
In the previous section we established that $(\sigma, M)$ models in the same universality class $[\sigma_\eff]$ (\autoref{def:universality}) are equivalent in the sense formalized in \autoref{prop:universal} and \autoref{corollary:xhatlaw}. In this section, we switch our focus to how difficult it is to sample universal M-densities. In particular, we formalize the intuition that $(\sigma_\eff \sqrt{M}, M) \in [\sigma_\eff]$ become more spherical by increasing $M$. The analysis also sheds some light on the geometry of M-densities as one increases $M$ (for a fixed $\sigma$). 

An important parameter characterizing the shape of log-concave densities is the condition number denoted by $\kappa$ which measures how elongated the density is~\citep[Section 1.4.1]{cheng2018underdamped}. The condition number appears in the mixing time analysis of log-concave densities, e.g. in the form $\kappa^2$ in~\citep{cheng2018underdamped}. Intuitively, for poorly conditioned densities one has to use a small step size and that will lead to long mixing times. 

For Gaussian densities $$p(x) = \mathcal{N}(x; \mu, \Sigma),$$ the condition number denoted by $\kappa$ is given by \< \kappa = \lambda_{\max}(F)/ \lambda_{\min}(F),\> where $\lambda(F)$ denotes the spectrum of the inverse covariance matrix $F \coloneqq \Sigma^{-1}$.

Next, we study the full spectrum of the corresponding $M$-density and give an expression for the condition number of $(\sigma, M)$ models. We assume without loss of generality a basis in $\mathbb{R}^d$ where the density is centered at $\mu=0$ and the covariance matrix is  diagonalized: $\Sigma_{ij} = \tau_i^{2} \delta_{ij}$. Therefore,
$$ p(x) = \prod_{i=1}^d \frac{1}{Z(\tau_i)}\exp{\left(-\frac{x_i^2}{2\tau_i^2}\right)},$$
and $\kappa = \tau_{\max}^2/\tau_{\min}^2$. Next we study the condition number for $(\sigma, M)$ models. The case $M=1$ is simple since $$F = (\Sigma + \sigma^2 I_d)^{-1},$$ therefore \<\kappa(\sigma, 1) = \frac{\tau_{\max}^2 + \sigma^2}{\tau_{\min}^2 + \sigma^2}.\>

We switch to studying the full spectrum of the covariance matrix for $(\sigma, M)$ models where $M>1$. We start with the expression for $p(\y)$ given below up to a normalizing constant for the general M-density defined by the noise levels $(\sigma_1,\sigma_2, \dots, \sigma_M)$:
\<\begin{split}
	 p(\y) &\propto \int \prod_{m=1}^M \mathcal{N}(x;y_m,\sigma_m^2 I_d) \cdot \mathcal{N}(x;0,\Sigma) \;dx \\
	 &\propto \prod_{i=1}^d  \int \exp{\left(-\sum_{m=1}^M\frac{\left(y_{mi} - x_i\right)^2}{2\sigma_m^2}-\frac{x_i^2}{2\tau_i^2}\right)} dx_i \\
	 &= \prod_{i=1}^d \int \exp{\left(-\frac{(x_i-\alpha_i)^2}{2 \beta_i^2}-\gamma_i\right)} dx_i\\
	 &\propto \prod_{i=1}^d \exp(-\gamma_i).
\end{split}
\>
The expressions for $\alpha_i$, $\beta_i$, and $\gamma_i$ are given next by completing the square via matching second, first and zeroth derivative (in that order) of the left and right hand sides below
\< -\sum_{m=1}^M\frac{\left(y_{mi} - x_i\right)^2}{2\sigma_m^2}-\frac{x_i^2}{2\tau_i^2} = -\frac{(x_i-\alpha_i)^2}{2 \beta_i^2}-\gamma_i \>
evaluated at $x_i=0$. The following three equations follow:
\< \frac{1}{\beta_i^2}= \sum_{m=1}^M \frac{1}{\sigma_m^2} + \frac{1}{\tau_i^2} \hspace{1cm},  \>
\<\frac{\alpha_i}{\beta_i^2} = \sum_{m=1}^M \frac{y_{mi}}{\sigma_m^2} \Rightarrow \alpha_i = \omega_i^2 \sum_{m=1}^M \frac{y_{mi}}{\sigma_m^2}, \>
\< 
	 -\frac{\alpha_i^2}{2 \beta_i^2}-\gamma_i = -\sum_{m=1}^M \frac{y_{mi}^2}{2\sigma_m^2} \\ \Rightarrow \gamma_i = \sum_{m=1}^M \frac{y_{mi}^2}{2\sigma_m^2} - \frac{1}{2} \omega_i^2 \left(\sum_{m=1}^M \frac{y_{mi}}{\sigma_m^2}\right)^2,
  \>
where \< \omega_i^2 \coloneqq \left(\sum_{m=1}^M \frac{1}{\sigma_m^2} + \frac{1}{\tau_i^2} \right)^{-1}.\>
Therefore the energy function associated with the random variable $\y$ is given by:
\< f(\y) =  \sum_m \frac{\Vert y_m \Vert^2}{2\sigma_m^2}- \frac{1}{2} \sum_{i=1}^d \omega_i^2 \left(\sum_{m=1}^M \frac{y_{mi}}{\sigma_m^2}\right)^2. \>
The energy function can be written more compactly by introducing the matrix $\F$:
\< \begin{split} 
f(\y) &= \frac{1}{2} \langle \y, \F \y \rangle,\\ 
\F_{mi,m'i'} &= [ \sigma_m^{-2}(1-\omega_i^2 \sigma_m^{-2}) \delta_{mm'}- \omega_i^2 \sigma_m^{-2}  \sigma_{m'}^{-2}(1-\delta_{mm'}) ] \delta_{ii'}.
 \end{split}  \>
In words, the $Md \times Md$ dimensional matrix $\F$ is block diagonal with $d$ blocks of size $M\times M$. The blocks themselves capture the interactions between different measurements indexed by $m$ and $m'$. To study the spectrum of the covariance matrix, we next focus on $(\sigma, M)$ models, i.e., the permutation-invariant case where $\sigma_m = \sigma$ for all $m \in [M]$: 
\< 
\begin{split}
	\F_{mi,m'i'} &= [ \sigma^{-2}(1-\omega_i^2 \sigma^{-2}) \delta_{mm'}- \omega_i^2 \sigma^{-4} (1-\delta_{mm'}) ] \delta_{ii'},\\
	\omega_i^2 &= (M \sigma^{-2} + \tau_i^{-2})^{-1}.
\end{split}
\>
The $M\times M$ blocks of the matrix $\F$ have the form:
$$ \F_i = \begin{pmatrix}
  a_i & b_i & \dots & b_i \\
  b_i & a_i & \dots & b_i \\
  \vdots & 	& \ddots	 \\
  b_i & b_i &\cdots & a_i
 \end{pmatrix},
$$
where
\begin{align}
	a_i &= \sigma^{-2}(1-\omega_i^2 \sigma^{-2}),  \\
	b_i &= - \omega_i^2 \sigma^{-4} .
\end{align}
It is straightforward to find the $M$ eigenvalues of the matrix~$\F_i$: \begin{itemize}
	\item $M-1$ degenerate eigenvalues equal to $a_i-b_i$ corresponding to the eigenvectors 
$$\{(1, -1, 0, \dots, 0)^\top, (1, 0, -1, \dots, 0)^\top, \dots, (1, 0, 0, \dots, -1)^\top\},$$
	\item one eigenvalue equal to $a_i+(M-1)b_i$ corresponding to the eigenvector $(1, 1, \dots, 1)^\top$.
\end{itemize} 
Since $a_i>0,~ b_i<0$ we arrive at: $$
	\lambda_{\rm max}(\F) = a_i - b_i = \sigma^{-2} $$ which is $(M-1) d$ degenerate on the full matrix $\F$. The remaining $d$ eigenvalues are given by 
\<
	\lambda_i = a_i + (M-1) b_i \\
	= \sigma^{-2}(1-M \omega_i^2 \sigma^{-2}),\>
	the smallest of which is given by
\< \begin{split} \lambda_{\rm min}(\F) &= \sigma^{-2}(1-M \sigma^{-2} \omega_{\rm max}^2 ) = 	\sigma^{-2}\left(1-M\sigma^{-2} \frac{1}{M \sigma^{-2} + \tau_{\rm max}^{-2}}\right) =  \frac{\sigma^{-2}}{1+M\sigma^{-2} \tau_{\rm max}^2} .
 \end{split}
\>
It follows:
\< \kappa(\sigma,M) =  \lambda_{\rm max}(\F)/\lambda_{\rm min}(\F)=1+M\sigma^{-2} \tau_{\rm max}^2 . \>

\section{Experiments}
\label{sec:exp}

We conducted experiments on the CIFAR-10 dataset \citep{krizhevsky2009learning} of 32$\times$32 color images from 10 classes. The goal of these experiments is to empirically study the results of sampling from $\gps$ models in the same universality class as implied by our theoretical analysis in \autoref{sec:universal}.

\paragraph{Training.}
$\nu_{\theta}$ (from \autoref{eq:xhat_theta}) was parameterized using the ``U-Net'' used in recent work on generative modeling on this dataset \citep{dhariwal2021diffusion}. We set $\sigma_\eff=0.25$, by choosing $\sigma=1$ and $M=16$ for training the network.
Similar to the MDAE parametrization from \citet{saremi2022multimeasurement}, learning essentially involves training a denoising autoencoder with a mean squared loss.
For optimization, the Adagrad optimizer \citep{duchi2011adaptive} was used with a batch size of 128 and maximum 400 epochs of training. 
The learning rate was initialized \num{1e-6} and scheduled to linearly increase to \num{1.5} over \num{1e6} updates (though training terminated earlier).
During training, the FID score \citep{heusel2017gans} computed using samples from 125 parallel MCMC chains (400 samples each, resulting in 50,000 samples total) was monitored at regular intervals, and the model with the lowest FID score was selected as a form of early stopping.

\paragraph{Sampling Results.} Our sampling algorithm is based on the walk-jump sampling~\citep{saremi2019neural}  that samples noisy data using the learned score function with Langevin MCMC (walk) together with the Bayes estimator of clean data (jump). For Langevin MCMC we considered three different algorithms~\citep{sachs2017langevin, cheng2018underdamped, shen2019randomized}. We settled on the algorithm by~\citet{sachs2017langevin} early on as it was more reliable in the small-scale experiments that we performed (see \autoref{fig:gallery}(e) for a visual comparison to the randomized midpoint method by~\citet{shen2019randomized}). We set the step size $\delta= \sigma/2$ for all $(\sigma, M)$ models and did extensive experiments on tuning the friction parameter. The results are shown in \autoref{fig:fid} where the FID score is obtained averaged over 5 random seeds. In the algorithm by~\citet{sachs2017langevin} the friction parameter $\gamma$ only shows up in the form $\gamma_\eff = \gamma \cdot \delta$ which we call effective friction. This is especially important in our model since step sizes vary greatly between different $(\sigma, M)$ models. The best  results were obtained for $(0.25, 1)$ model with the FID of {\bf 14.15}. In addition, our results in \autoref{fig:gallery} are remarkable in qualitatively demonstrating fast mixing in long-run MCMC chains, where diverse classes are visited in a single chain: such fast-mixing MCMC chains on CIFAR-10 have not been reported in the literature.

\begin{figure}[h!]
    \centering
    \includegraphics[width=0.49\textwidth]{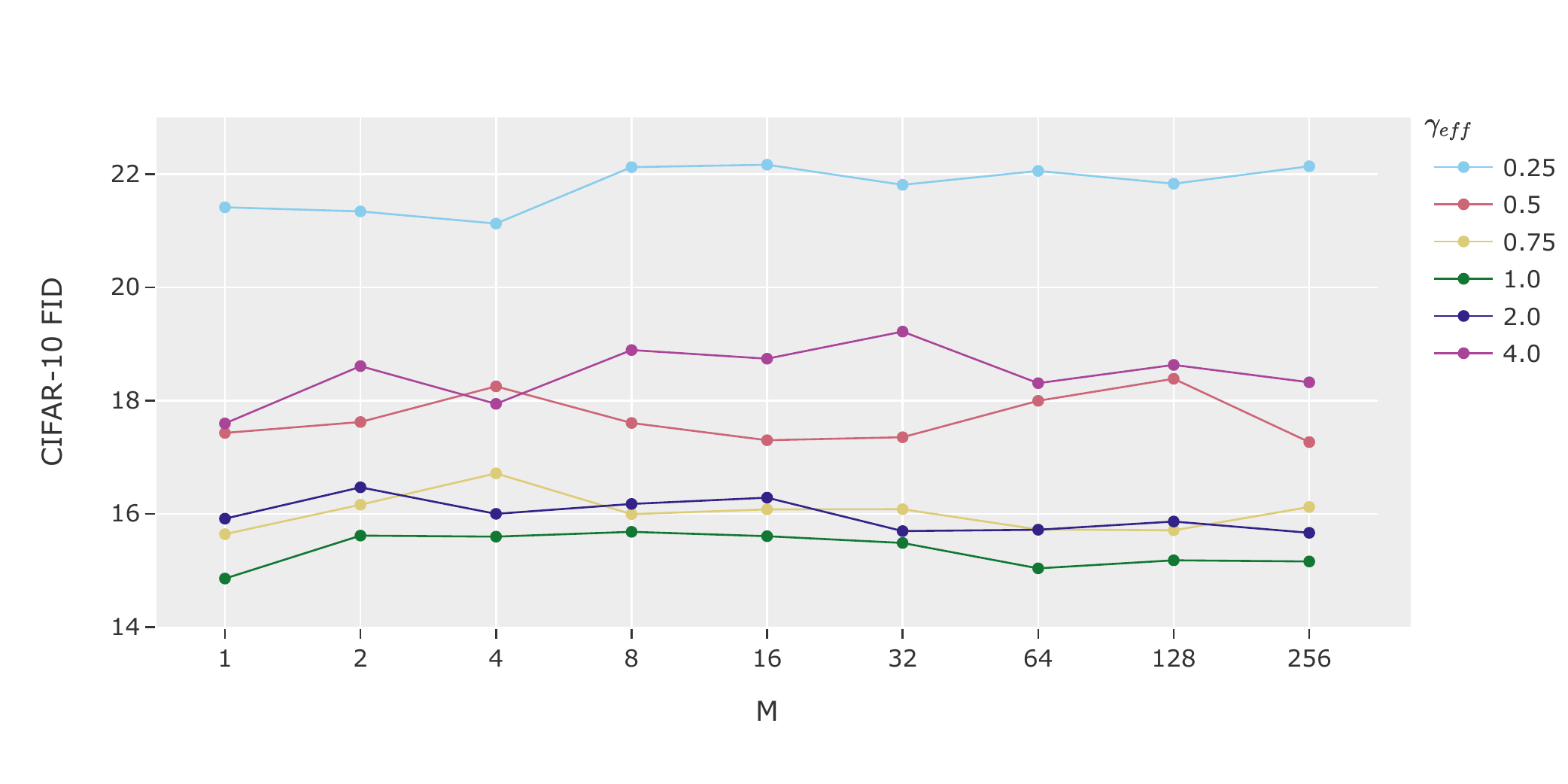}
    \includegraphics[width=0.49\textwidth]{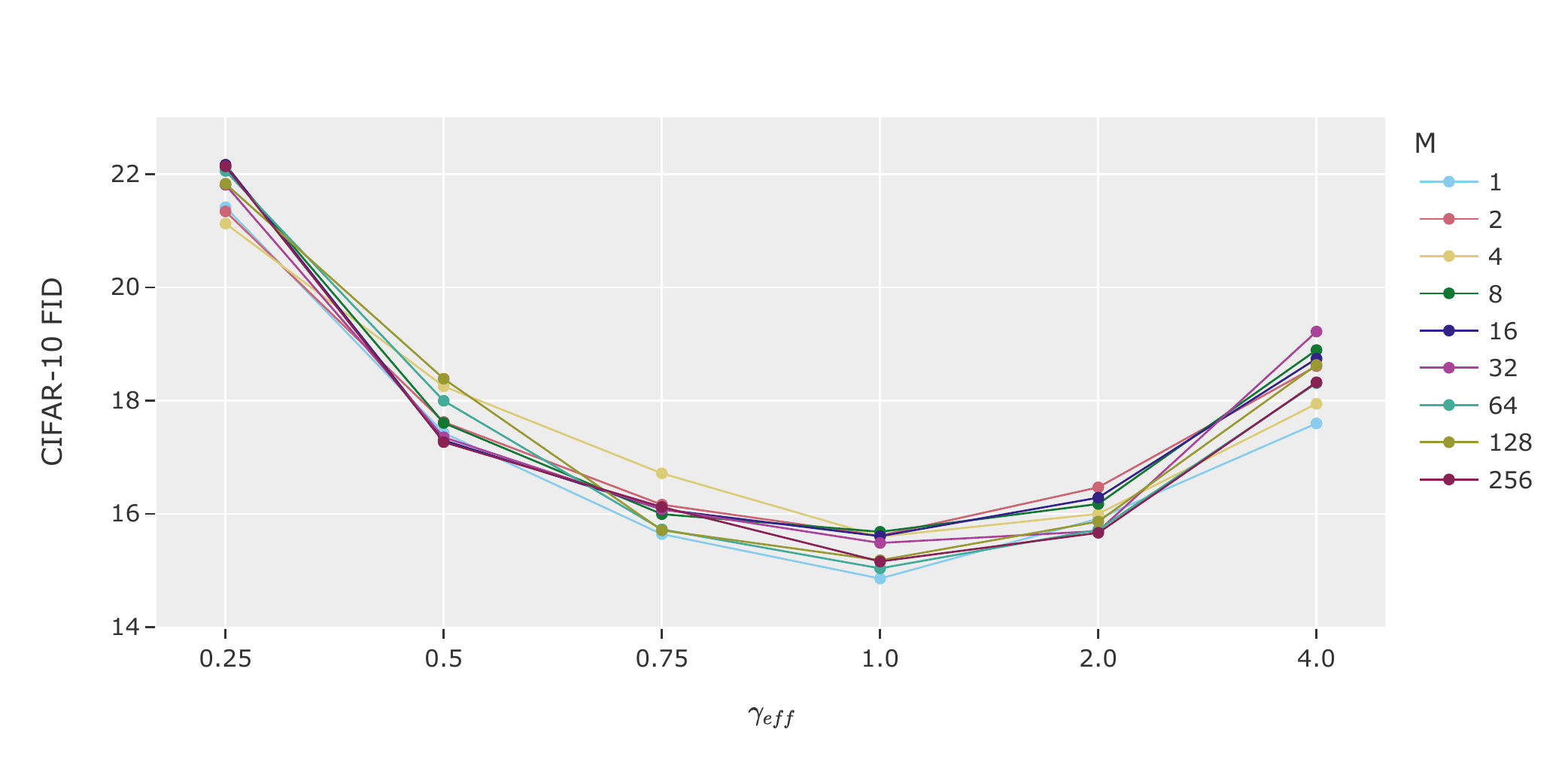}
    \caption{CIFAR-10 FID scores obtained when tuning the value of $\gamma_\eff$ for various values of $M$, for a model trained with $\sigma_\eff=0.25$.}
    \label{fig:fid}
\end{figure}

\begin{figure}[t!] 
\begin{center}
\begin{subfigure}[$(\sfrac{1}{4},1)$, $\delta=\sfrac{1}{8}$, $\gamma=8$]
 {\includegraphics[width=0.9\textwidth]{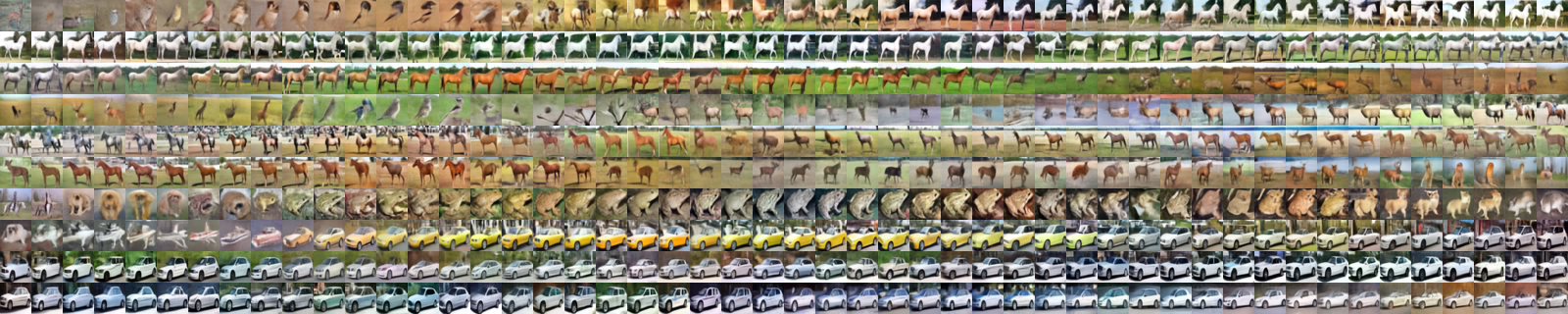}}
\end{subfigure}
\begin{subfigure}[$(1,16)$, $\delta=\sfrac{1}{2}$, $\gamma=2$]
 {\includegraphics[width=0.9\textwidth]{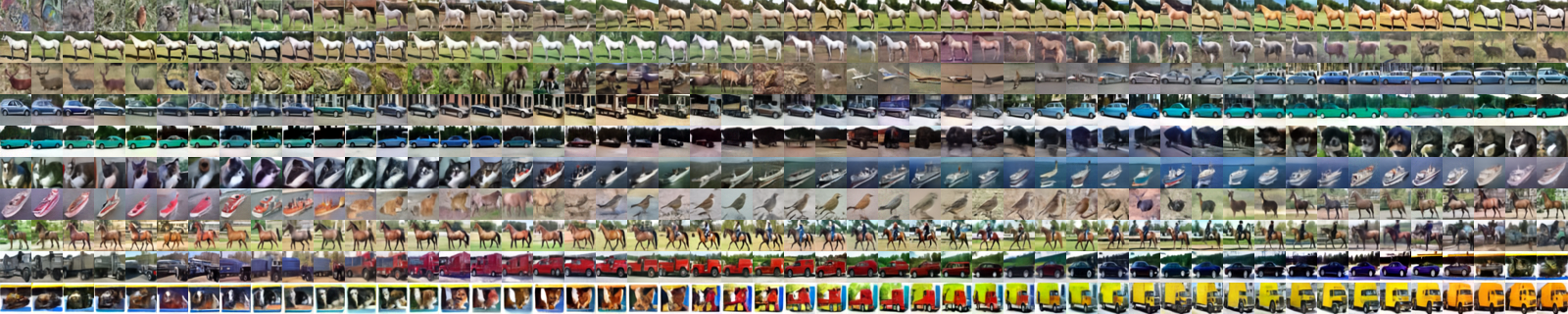}}
\end{subfigure}
\begin{subfigure}[$(2,64)$,  $\delta=1$, $\gamma=1$]
 {\includegraphics[width=0.9\textwidth]{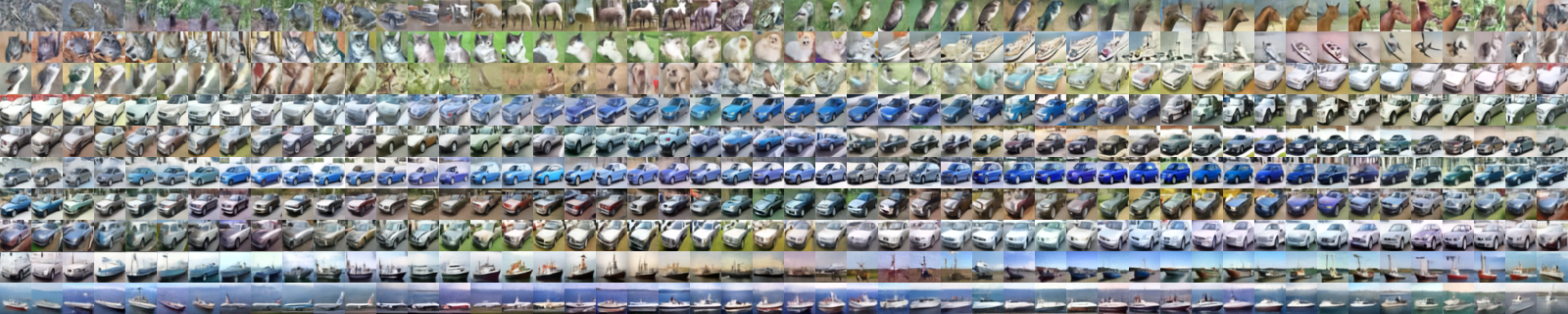}}
\end{subfigure}
\begin{subfigure}[$(4,256)$, $\delta=2$, $\gamma=0.5$]
 {\includegraphics[width=0.9\textwidth]{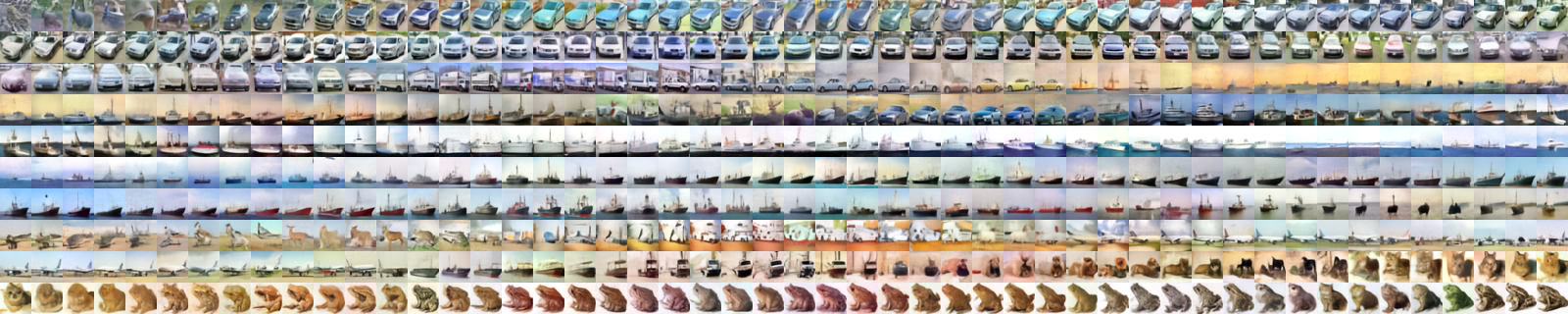}}
\end{subfigure}
\begin{subfigure}[$(1, 16)$, $\delta=\sfrac{1}{2}$, $\gamma=2$]
{\includegraphics[width=0.9\textwidth]{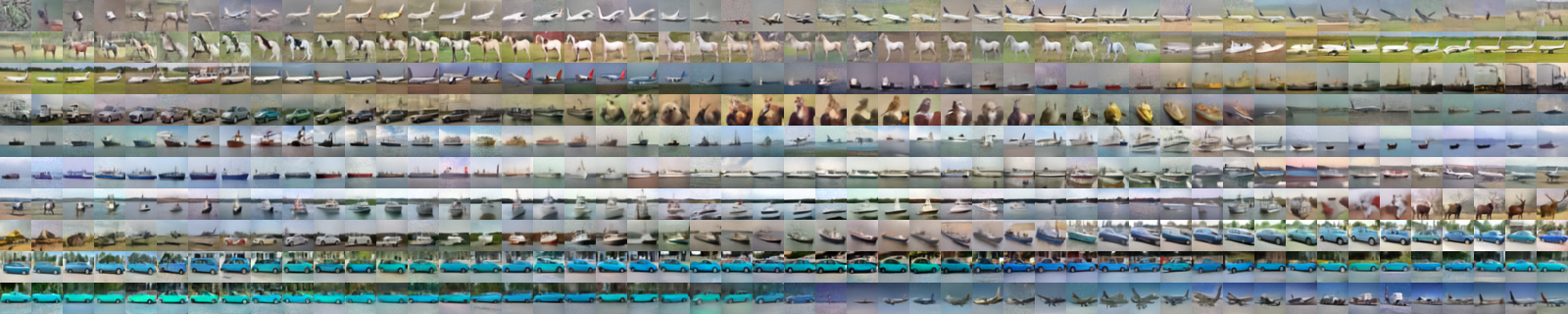}}
\end{subfigure}
\caption{Examples of long-run MCMC chains on CIFAR-10 dataset for $(\sigma,M)$ models in the $[\sigma_\eff = 0.25]$ universality class. {\bf Each panel represents a single MCMC chain.} Only 20 steps are taken per image, starting from noise, a total of 10\,K steps (viewed left-to-right, top-to-bottom). We set $\delta = \sigma/2$, and we set the friction $\gamma$ from \citet{sachs2017langevin}, panels (a)-(d), such that $\gamma_\eff \coloneqq \gamma \delta =1$.  In the bottom panel we show the performance of randomized midpoint method~\citep{shen2019randomized} used for Langevin MCMC in the walk-jump sampling. The random seed is fixed across runs. Best seen zoomed on a computer screen.}
\label{fig:gallery}
\end{center}
\end{figure}

\clearpage

\section{Conclusion}
This work was primarily concerned with the theoretical understanding of smoothing methods using factorial kernels proposed by~\citet{saremi2022multimeasurement}, where in particular we focused on the permutation-invariant case in which the model is defined by a single noise scale. We showed such models are grouped into universality classes in which the densities can be easily mapped to each other, and we introduced the $\gps$ parametrization to utilize that. Theoretically, the models that belong to the same universality class should have very different sampling properties and we had an analysis of that here, focused on studying the condition number. 

Our experimental results on CIFAR-10 were surprising on two fronts: (i) We achieved low FID scores which have been argued to be not feasible for denoising models with a single noise scale. In fact, the research on generative models based on denoising autoencoders (DAE) came to a halt around 2014 with the invention of GANs~\citep{goodfellow2014generative}, and we were ourselves surprised by a simple  model such as $(0.25,1)$ outperforming BigGAN~\citep{brock2018large} on the CIFAR-10 challenge. Note that, $(0.25,1)$ is essentially a DAE with an empirical Bayes interpretation~\citep{saremi2019neural}. (ii) We also found it surprising that our experimental results did not show any benefit for larger $M$ models, but that needs more investigation in future research. In particular, there might exist more ``clever samplers'' that need to be invented for exploiting the structure of $(\sigma, M)$ models.

\section*{Acknowledgement}
We would like to thank Ruoqi Shen for communication regarding their randomized midpoint method.

\bibliography{biblio}
\bibliographystyle{apalike}

\end{document}